\def\year{2022}\relax
\documentclass[letterpaper]{article} 
\usepackage{aaai22}  
\usepackage{times}  
\usepackage{helvet}  
\usepackage{courier}  
\usepackage[hyphens]{url}  
\usepackage{graphicx} 
\def\UrlFont{\rm}  
\usepackage{natbib}  
\usepackage{caption} 
\DeclareCaptionStyle{ruled}{labelfont=normalfont,labelsep=colon,strut=off} 
\usepackage{algorithm}
\usepackage{algorithmic}
\usepackage{mathtools}
\usepackage{amsmath}
\usepackage{amssymb}
\usepackage{amsfonts}
\usepackage{amsthm}
\usepackage{bibunits}
\usepackage{paralist}

\newtheorem{thm}{Theorem}[subsection]
\newtheorem{definition}{Definition}[subsection]
\usepackage{newfloat}
\usepackage{listings}
\floatstyle{ruled}
\newfloat{listing}{tb}{lst}{}
\floatname{listing}{Listing}
\title{Cooperative Multi-Agent Fairness and Equivariant Policies}
\author{
    Niko A. Grupen \textsuperscript{\rm 1},
    Bart Selman \textsuperscript{\rm 1},
    Daniel D. Lee \textsuperscript{\rm 2},
}
\affiliations {
    \textsuperscript{\rm 1} Cornell University, Ithaca, NY\\
    \textsuperscript{\rm 2} Cornell Tech, New York, NY\\
}

\begin{document}

\maketitle

\begin{abstract}
    We study fairness through the lens of cooperative multi-agent learning. Our work is motivated by empirical evidence that naive maximization of team reward yields unfair outcomes for individual team members. To address fairness in multi-agent contexts, we introduce team fairness, a group-based fairness measure for multi-agent learning. We then prove that it is possible to enforce team fairness during policy optimization by transforming the team's joint policy into an equivariant map. We refer to our multi-agent learning strategy as \textit{Fairness through Equivariance} (Fair-E) and demonstrate its effectiveness empirically. We then introduce Fairness through Equivariance Regularization (Fair-ER) as a soft-constraint version of Fair-E and show that it reaches higher levels of utility than Fair-E and fairer outcomes than non-equivariant policies. Finally, we present novel findings regarding the fairness-utility trade-off in multi-agent settings; showing that the magnitude of the trade-off is dependent on agent skill.
\end{abstract}

\section{Introduction}
\label{sec_intro}
Algorithmic fairness is an increasingly important sub-domain of AI. As statistical learning algorithms continue to automate decision-making in crucial areas such as lending \cite{fuster2020predictably}, healthcare \cite{potash2015predictive}, and education \cite{dorans2016fairness}, it is imperative that the performance of such algorithms does not rely upon sensitive information pertaining to the individuals for which decisions are made (e.g. race, gender). Despite its growing importance, fairness research has largely targeted prediction-based problems, where decisions are made for one individual at one time \cite{mitchell2018prediction}. Though recent studies have extended fairness to the multi-agent case \cite{jiang2019learning}, such work primarily considers social dilemmas in which team utility is in obvious conflict with the local interests of each team member \cite{leibo2017multi, rapoport1974prisoner, van2013psychology}.

Many real-world problems, however, must weigh the fairness implications of team behavior in the presence of a single overarching goal. In-line with recent work that has highlighted the importance of leveraging multi-agent learning to study socio-economic challenges such as taxation, social planning, and economic policy \cite{zheng2020ai}, we posit that understanding the range of team behavior that emerges from single-objective utility maximization is crucial for the development of fair multi-agent systems. For this reason, we study fairness in the context of \textit{cooperative multi-agent settings}. Cooperative multi-agent fairness differs from traditional game-theoretic interpretations of fairness (e.g. resource allocation \cite{elzayn2019fair, zhang2014fairness}, social dilemmas \cite{leibo2017multi}) in that it seeks to understand the fairness implications of emergent coordination learned by multi-agent teams that are bound by a shared reward. Cooperative multi-agent fairness therefore reframes the question---``Will agents cooperate or defect, given the choice between local and team interests?"---to a related but novel question---``Given the incentive to work together, do agents learn to coordinate effectively and fairly?"

Experimentally, we target pursuit-evasion (i.e. predator-prey) as a test-bed for cooperative multi-agent fairness. Pursuit-evasion allows us to simulate a number of important components of socio-economic systems, including: 
\begin{inparaenum}[(i)]
    \item Shared objectives: the overarching goal of pursuers is to capture an evader;
    \item Agent skill: the speed of the pursuers relative the evader serves as a proxy for skill;
    \item Coordination: success requires sophisticated cooperation by the pursuers.
\end{inparaenum}
Using pursuit-evasion, we study the fairness implications of behavior that emerges under variations of these ``socio-economic" parameters. Similar to prior work \cite{lowe2017multi, mordatch2018emergence, grupen2020low}, we cast pursuit-evasion as a multi-agent reinforcement learning (RL) problem. 

Our first result highlights the importance of shared objectives to cooperation. In particular, we compare policies learned when pursuers share in team success (mutual reward) to those learned when pursuers do not share reward (individual reward). We find that sophisticated coordination only emerges when pursuers are bound by mutual reward. Given individual reward, pursuers are not properly incentivized to work together. However, though mutual reward aides coordination, it does not specify how to coordinate fairly. In our experiments, we find that naive, unconstrained maximization of mutual reward yields unfair individual outcomes for cooperative teammates. In the context of pursuit-evasion, the optimal strategy is a form of \textit{role assignment}---the majority of pursuers act as supporting agents, shepherding the evader to one designated ``capturer" agent. Solving this issue is the subject of the rest of our analysis.

Addressing this form of unfair emergent coordination requires connecting fairness to multi-agent learning settings. To do this, we first introduce \textit{team fairness}, a group-based fairness measure inspired by demographic parity \cite{dwork2012fairness, feldman2015certifying}. Team fairness requires the distribution of a team's reward to be equitable across sensitive groups. We then show that it is possible to enforce team fairness during policy optimization by transforming the team's joint policy into an equivariant map. We prove that equivariant policies yield fair reward distributions under assumptions of agent homogeneity. We refer to our multi-agent learning strategy as \textit{Fairness through Equivariance} (Fair-E) and demonstrate its effectiveness empirically in pursuit-evasion experiments.

Despite achieving fair outcomes, Fair-E represents a binary switch---one can either choose fairness (at the expense of utility) or utility (at the expense of fairness). In many cases, however, it is advantageous to modulate between fairness and utility. To this end, we introduce a soft-constraint version of Fair-E that incentivizes equivariance through regularization. We refer to this method as \textit{Fairness through Equivariance Regularization} (Fair-ER) and show that it is possible to tune fairness constraints over multi-agent policies by adjusting the weight of equivariance regularization. Moreover, we show empirically that Fair-ER reaches higher levels of utility than Fair-E while achieving fairer outcomes than non-equivariant policy learning.

Finally, as in both prediction-based settings \cite{corbett2017algorithmic, zhao2019inherent} and in traditional multi-agent variants of fairness \cite{okun2015equality, le1990equity}, it is important to understand the ``cost" of fairness. We present novel findings regarding the fairness-utility trade-off for cooperative multi-agent settings. Specifically, we show that the magnitude of the trade-off depends on the skill level of the multi-agent team. When agent skill is high (making the task easier to solve), fairness comes with no trade-off in utility, but as skill decreases (making the task more difficult), gains in team fairness are increasingly offset by decreases in team utility.
\newline

\noindent \textbf{Preview of contributions}
In sum, our work offers the following contributions:
\begin{itemize}
    \item We show that mutual reward is critical to multi-agent coordination. In pursuit-evasion, agents trained with mutual reward learn to coordinate effectively, whereas agents trained with individual reward do not.
    \item We connect fairness to cooperative multi-agent settings. We introduce team fairness as a group-based fairness measure for multi-agent teams that requires equitable reward distributions across sensitive groups.
    \item We introduce Fairness through Equivariance (Fair-E), a novel multi-agent strategy leveraging equivariant policy learning. We prove that Fair-E achieves fair outcomes for individual members of a cooperative team.
    \item We introduce Fairness through Equivariance Regularization (Fair-ER) as a soft-constraint version of Fair-E. We show that Fair-ER reaches higher levels of utility than Fair-E while achieving fairer outcomes than non-equivariant learning.
    \item We present novel findings regarding the fairness-utility trade-off for cooperative settings. Specifically, we show that the magnitude of the trade-off depends on agent skill---when agent skill is high, fairness comes for free; whereas with lower skill levels, fairness is increasingly expensive.
\end{itemize}

\section{Related work}
\label{sec_related_work}
At a high-level, the prediction-based fairness literature can be split into two factions: individual fairness and group fairness. Introduced by \citet{dwork2012fairness}, individual fairness posits that two individuals with similar features should be classified similarly (i.e. similarity in feature-space implies similarity in decision-space). Such approaches rely on task-specific distance metrics with which similarity can be measured \cite{barocas2019fairmlbook, chouldechova2018frontiers}. Group fairness, on the other hand, attempts to achieve outcome consistency across sensitive groups. This idea has given rise to a number of methods such as statistical/demographic parity \cite{feldman2015certifying, johndrow2019algorithm, kamiran2009classifying, zafar2017fairness}, equality of opportunity \cite{hardt2016equality}, and calibration \cite{kleinberg2016inherent}. Recent work has extended fairness to the RL setting to consider the feedback effects of decision-making \cite{jabbari2017fairness, wen2021algorithms}.

In multi-agent systems, fairness is typically studied in game-theoretic settings in which individual payoffs and overall group utility are in obvious conflict \cite{de2005priority, de2007considerations}---such as resource allocation \cite{elzayn2019fair, zhang2014fairness} and social dilemmas \cite{leibo2017multi, rapoport1974prisoner, van2013psychology}. In multi-agent RL settings, these tensions have been addressed through myriad techniques, including reward shaping \cite{peysakhovich2017prosocial}, intrinsic reward \cite{wang2018evolving}, parameterized inequity aversion \cite{hughes2018inequity}, and hierarchical learning \cite{jiang2019learning}. Also related is the Shapley value: a method for sharing surplus across a coalition based on one’s contributions to the coalition \cite{shapley201617}. Shapley value-based credit assignment techniques have recently been shown to stabilize learning and achieve fairer outcomes when incorporated into the multi-agent RL problem \cite{wang2020shapley, li2021shapley}.

Our work differs from this prior work in two key ways. First we target fully-cooperative multi-agent settings \cite{hao2016fairness} in which fairness implications emerge naturally in the presence of a single overarching goal (i.e. mutual reward). In this fully-cooperative setting, individual and team incentives are not in obvious conflict. Our motivation for studying fully-cooperative team objectives follows from recent work that highlights the role of multi-agent learning in real-world problems characterized by shared objectives, including taxation and economic policy \cite{zheng2020ai}. Moreover, we study modifications to the utility-maximization objective that yield fairer outcomes by incentivizing agents to change their behavior, rather than redistributing outcomes after-the-fact. Most relevant is \citet{siddique2020learning} and \citet{zimmer2020learning}, which introduce a class of algorithms that successfully achieve fair outcomes for multi-agent teams through pre-defined social welfare functions that encode specific fairness principles. Our work, conversely, introduces task-agnostic methods for incentivizing fairness through both hard-constraints on agent policies and soft-constraints (i.e. regularization) \cite{liu2018delayed} on the RL objective.

Finally, discussion of the fairness-utility (or fairness-efficiency) trade-off has a long history in game-theoretic multi-agent settings \cite{okun2015equality, le1990equity, bertsimas2012efficiency, joe2013multiresource, bertsimas2011price} and is also prevalent throughout the prediction-based fairness literature \cite{menon2018cost}. Existing work has shown both theoretically \cite{calders2009building, kleinberg2016inherent, zhao2019inherent} and empirically \cite{dwork2012fairness, feldman2015certifying, kamiran2009classifying, lahoti2019operationalizing, pannekoek2021investigating} that gains in fairness come at the cost of utility. Our discussion of the fairness-utility trade-off is most similar to \citet{corbett2017algorithmic} in this regard, as we study the trade-off through the lens of constrained vs. unconstrained optimization. However, we take this trade-off a step further, outlining a relationship between fairness, utility, and agent skill that is not present in prior work.

\section{Preliminaries}
\label{sec_preliminaries}

\paragraph{Markov games}
A Markov game is a multi-agent extension of the Markov decision process (MDP) formalism \cite{littman1994markov}. For $n$ agents, it is represented by a state space $\mathcal{S}$, joint action space $\boldsymbol{\mathcal{A}} = \{\mathcal{A}_1, ... , \mathcal{A}_n\}$, joint observation space $\boldsymbol{\mathcal{O}} = \{\mathcal{O}_1, ... , \mathcal{O}_n\}$, transition function $\mathcal{T}:\mathcal{S} \times \boldsymbol{\mathcal{A}} \rightarrow \mathcal{S}$, and joint reward function $\boldsymbol{r}$. Following multi-objective RL \cite{zimmer2020learning}, we define a vectorial reward $\boldsymbol{r}:\mathcal{S} \times \boldsymbol{\mathcal{A}} \rightarrow \mathbb{R}^n$ with each component $r_i$ representing agent $i$'s contribution to $\boldsymbol{r}$. Each agent $i$ is initialized with a policy $\pi_i:\mathcal{O}_i \rightarrow \mathcal{A}$ (or deterministic policy $\mu_i$) from which it selects actions and an action-value function $Q_i: \mathcal{S} \times \mathcal{A}_i \rightarrow \mathbb{R}$ with which it judges the value of state-action pairs. Following action selection, the environment transitions from its current state $s_t$ to a new state $s_{t+1}$, as governed by $\mathcal{T}$, and produces a reward vector $\boldsymbol{r}_t$ indicating the strength or weakness of the group’s decision-making. In the episodic case, this process continues for a finite time horizon $T$, producing a trajectory $\tau = (s_1, \boldsymbol{a_1}, ..., s_{T-1}, \boldsymbol{a_{T-1}}, s_T)$ with probability:
\begin{equation}
    \label{eq_traj_prob}
    P(\tau) = P_\emptyset(s_1)\prod_{t=1}^T  P(s_{t+1} | s_t, a_t) \pi(a_t|s_t)
\end{equation}
\noindent where $P_\emptyset$ is a special distribution specifying the likelihood of each ``start" state.

\paragraph{Deep deterministic policy gradients}
Deep Deterministic Policy Gradients (DDPG) is an off-policy actor-critic algorithm for policy gradient learning in continuous action spaces \cite{lillicrap2015continuous}. DDPG leverages the deterministic policy gradient theorem \cite{silver2014deterministic}, which asserts that it is possible to find an optimal deterministic policy ($\mu_\phi$, with parameters $\phi$), with respect to a Q-function ($Q_{\omega}$, with parameters $\omega$), for the RL objective:
\begin{equation}
    \label{eq_ddpg_obj}
    J(\phi) = \mathbb{E}_s[Q_{\omega}(s,a) |_{s=s_t, a=\mu_{\phi}(s_t)} ]
\end{equation}
\noindent by performing gradient ascent with respect to the following gradient:
\begin{equation}
    \label{eq_ddpg_grad}
    \nabla_\phi J(\phi) = \mathbb{E}_s [\nabla_a Q_{\omega}(s,a)|_{s=s_t, a=\mu(s_t)}\nabla_\phi \mu(s)|_{s=s_t}]
\end{equation}
\noindent under mild conditions that confirm the existence of gradients $\nabla_a Q_{\omega}(s,a)$ and $\nabla_\phi \mu(s)$. For critic updates, DDPG follows batched TD-control, where the Q-function minimizes the loss function:
\small
\begin{equation}
    L(\omega) = \underset{s, a, r, s'}{\mathbb{E}} \big[\big(Q_\omega(s, a) - (r(s,a) + \gamma Q_\omega(s', \mu_\phi(s')))\big)^2\big]
\end{equation}
\normalsize
\noindent where ($s, a, r, s'$) are transition tuples sampled from a replay buffer. In this work, agents learn in a decentralized manner, each performing DDPG updates individually.

\paragraph{Fairness}
Prediction-based fairness considers a population of $n$ individuals (indexed $i = 1, ..., n$), each described by variables $v_i$ (i.e. features or attributes), which are separated into sensitive variables $z_i$ and other variables $x_i$. Variables $v_i$ are used to predict (typically binary) outcomes $y_i \in Y$ by estimating the conditional probability $P[Y=1 | V=v_i]$ through a scoring function $\psi:\mathcal{V} \rightarrow \{0,1\}$. Outcomes in turn yield decisions by applying a decision rule $\delta(v_i) = f(\psi(v_i))$. For example, in a lending scenario, a classifier may use $v_i$ to predict whether an individual $i$ will default on ($y_i = 0$) or repay ($y_i = 1$) his/her loan, which informs the decision to deny ($d_i = 0$) or approve ($d_i = 1$) the individual's loan application \cite{mitchell2018prediction}. Of particular relevance is group-based fairness, which examines how well outcome ($Y$) and decision ($D$) consistency is preserved across sensitive groups ($Z$) \cite{feldman2015certifying, zafar2017fairness}. We highlight the group-based measure of demographic parity, which requires that $D \perp Z$ or, equivalently, that $P[D=1|Z=z] = P[D=1|Z=z']$ for all $z, z'$ where $z \neq z'$.

\paragraph{Mutual Information} \label{sec_mi}
Given random variables $X_1 {\sim} P_{X_1}$ and $X_2 {\sim} P_{X_2}$ with joint distribution $P_{X_1 X_2}$, mutual information is defined as the Kullback-Leibler (KL-) divergence between the joint $P_{X_1X_2}$ and the product of the marginals $P_{X_1} \otimes P_{X_2}$:
\begin{equation}
    I(X_1;X_2) := D_{KL}(P_{X_1 X_2} || P_{X_1} \otimes P_{X_2}) \label{eq_dkl} \\
\end{equation}
\noindent Mutual information quantifies the dependence between $X_1$ and $X_2$ where, in Equation \ref{eq_dkl}, larger divergence represents stronger dependence. Importantly, mutual information can also be represented as the decrease in entropy of $X_1$ when introducing $X_2$:
\begin{equation}
    I(X_1;X_2) := H(X_1) - H(X_1|X_2)
    \label{eq_mi_ent}
\end{equation}

\paragraph{Equivariance}
Let $g_1$ and $g_2$ be G-sets of a group $G$ and $\sigma$ be a symmetry transformation over $G$. Then a function $f:g_1 \rightarrow g_2$ is equivariant with respect to $\sigma$ if the commutative relationship $f(\sigma \cdot x) = \sigma \cdot f(x)$ holds. Equivariance in the context of RL implies that separate policies will take the same actions under permutations of state space.
\begin{figure*}[t!]
    \centering
    \includegraphics[width=0.9\textwidth]{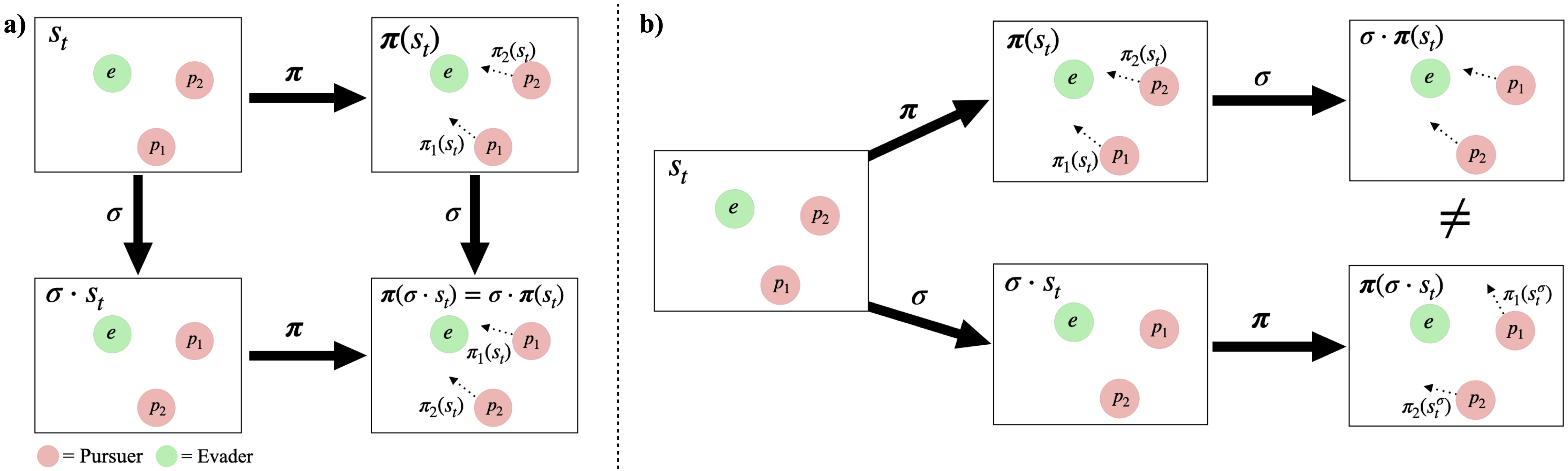}
    \caption{Snapshot of a pursuit-evasion game. Pursuers $p_1$ and $p_2$ (red) chase an evader $e$ (green) with the goal of capturing it. Given a state $s_t$, each pursuer selects its next heading (dotted arrow) from its policy, yielding the joint policy $\boldsymbol{\pi} {=} \{\pi_1(s_t), \pi_2(s_t)\}$. Evader action selection is omitted for clarity. \textbf{a)} For an equivariant joint policy, applying the transformation $\sigma$ to the state $s_t$ (producing $\sigma \cdot s_t$ which, in this example, swaps the positions of $p_1$ and $p_2$) and running the policy $\boldsymbol{\pi}(\sigma \cdot s_t)$ is equivalent to running the joint policy first and transforming the joint action afterwards (i.e. the commutative relationship $\boldsymbol{\pi}(\sigma \cdot s) = \sigma \cdot \boldsymbol{\pi}(s)$ holds). \textbf{b)} This commutative relationship does not necessarily hold for non-equivariant policies.}
    \label{fig_equivariance}
\end{figure*}
\paragraph{Pursuit-evasion}
Pursuit-evasion (i.e. predator-prey) is a classic setting for studying multi-agent coordination \cite{isaacs1999differential}. A pursuit-evasion game is defined between $n$ pursuers $\{p_1, ..., p_n \}$ and a single evader $e$. At any time $t$, an agent $i$ is described by its current position and heading $q_i^t = [x, y, \theta]_i$ and the environment is described by the position and heading of all agents $s_t = \{q_{p_1}, ..., q_{p_n}, q_e \}$. Upon observing $s_t$, each agent selects its next heading $\theta_i$ as an action. The chosen heading is pursued at the maximum allowed speed for each agent ($\lvert \vec{v}_p\rvert$ for the pursuers, $\lvert \vec{v}_e\rvert$ for the evader). We assume the evader to be part of the environment, defined by the potential-field policy:
\begin{equation}
    \label{eq_evader_objective}
    U(\theta_e) = \sum_i \bigg(\frac{1}{r_i}\bigg) \cos(\theta_e - \tilde{\theta}_i)
\end{equation}
where $r_i$ and $\tilde{\theta}_i$ are the L2-distance and relative angle between the evader and the $i$-th pursuer, respectively, and $\theta_e$ is the heading of the evader. Intuitively, $U(\theta_e)$ pushes the evader away from pursuers, taking the largest bisector between any two when possible. The goal of the pursuers---to capture the evader as quickly as possible---is mirrored in the reward function, where $r(s_t, a_t) {=} 50.0$ if the evader is captured and $r(s_t, a_t) {=} -0.1$ otherwise. Note that $\lvert \vec{v}_p \rvert$ serves as a proxy for agent skill level. When $\lvert \vec{v}_p \rvert > \lvert \vec{v}_e \rvert$, pursuers are skilled enough to capture the evader on their own, whereas $\lvert \vec{v}_p \rvert \leq \lvert \vec{v}_e \rvert$ requires that pursuers work together.

\section{Method}
\label{sec_fair_marl}
In this section, we present a novel interpretation of fairness for cooperative multi-agent teams. We then introduce our proposed method---Fairness through Equivariance (Fair-E)---and prove that it yields fair outcomes. Finally, we present Fairness through Equivariance Regularization (Fair-ER) as a soft-constraint version of Fair-E. 

\paragraph{Notation}
Let $n$ be the number of agents in a cooperative team. We describe each agent $i$ by variables $v_i = (z_i, x_i)$, consisting of sensitive variables $z_i \in Z$ and non-sensitive variables $x_i \in X$. In team settings, we define non-sensitive variables $x_i$ to be any variables that affect agent $i$’s performance on the team; such as maximum speed. Other variables that should not impact team performance, such as an agent’s identity or belonging to a minority group, are defined as sensitive variables $z_i$. We define fairness in terms of reward distributions $R$---where each $\boldsymbol{r} \in R$ is a vectorial team reward and each component $r_i$ is agent $i$'s contribution to $\boldsymbol{r}$. In the following definitions, let $I(R;Z)$ be the mutual information between reward distributions $R$ and sensitive variables $Z$.

\subsection{Team fairness}
\label{sec_team_fairness}
 We now define team fairness, a group-based fairness measure for multi-agent learning.
\begin{definition}[Exact Team Fairness]
    \label{def_exact_fairness}
    A set of cooperative agents achieves exact team fairness if $\textrm{I}(R;Z) = 0$.
\end{definition}

\begin{definition}[Approximate Team Fairness]
    \label{def_approx_fairness}
     A set of cooperative agents achieves approximate team fairness if $\textrm{I}(R;Z) \leq \epsilon$ for some $\epsilon > 0$.
\end{definition}

\noindent Team fairness connects cooperative multi-agent learning to group-based fairness, as $\textrm{I}(R;Z) = 0$ is equivalent to requiring $R \perp Z$ \cite{barocas2019fairmlbook}.

\subsection{Fairness through equivariance}
\label{sec_fairness_equivariance}
To enforce team fairness during policy optimization, we introduce a novel multi-agent learning strategy. The key to our approach is equivariance: by enforcing parameter symmetries \cite{ravanbakhsh2017equivariance} in each agent $i$'s policy network $\pi_{\phi_i}$, we show that equivariance propagates through the multi-agent RL problem. In particular, we show that the joint policy $\boldsymbol{\pi} = \{\pi_{\phi_1}, ..., \pi_{\phi_n} \}$ is an equivariant map with respect to permutations over state and action space. Further, we show that equivariance in policy-space begets equivariance in trajectory-space; namely, the terminal state $s_T$ following a multi-agent trajectory is equivariant to that trajectory's initial state $s_1$. Finally, we prove that equivariance in multi-agent policies and trajectories yields exact team fairness. A comparison of equivariant vs. non-equivariant joint policies is provided in Figure \ref{fig_equivariance}.

In the proofs that follow, we assume: (i) homogeneity across agents on the team---i.e. agents are identical in their non-sensitive variables $x$; (ii) the distribution of agent positions satisfies exchangeability. Finally, though our derivations utilize general (stochastic) policies, we provide equivalent proofs for deterministic policies in Appendix \ref{apdx_deterministic_proofs}.
    
\begin{thm}
    \label{thm_policy_eqv_map}
    If individual policies $\pi_{\phi_i}$ are symmetric, then the joint policy $\boldsymbol{\pi} = \{\pi_{\phi_1}, ..., \pi_{\phi_n} \}$ is an equivariant map.
\end{thm}
\begin{proof}
    Let $\sigma$ be a permutation operator that, when applied to a vector (such as a state $s_t$ or action $\boldsymbol{a}_t$), produces a permuted vector ($\sigma \cdot s_t = s^\sigma_t$ or $\sigma \cdot \boldsymbol{a}_t = \boldsymbol{a}^\sigma_t$, respectively). Under parameter symmetry (i.e. $\phi_1$=$\phi_2{= \cdots =}\phi_n$), we have:
    \begin{equation}
        \label{eq_equivariant_policy}
        \boldsymbol{\pi}(\sigma \cdot s) 
        = \boldsymbol{\pi}(s^{\sigma}) 
        = \boldsymbol{a}^\sigma 
        = \sigma \cdot \boldsymbol{a} 
        = \sigma \cdot \boldsymbol{\pi}(s)
    \end{equation}
    \noindent where the commutative relationship $\boldsymbol{\pi}(\sigma \cdot s) = \sigma \cdot \boldsymbol{\pi}(s)$ implies that $\boldsymbol{\pi}$ is an equivariant map. Commutativity here is crucial---Equation \eqref{eq_equivariant_policy} and therefore Theorems \ref{thm_states_eqv} and \ref{thm_sym_fair} do not hold for non-equivariant policies (see Figure \ref{fig_equivariance}b).
\end{proof}

\begin{thm}
    \label{thm_states_eqv}
    Let $p^\pi(s \rightarrow s', k)$ be the probability of transitioning from state $s$ to state $s'$ in $k$ steps \cite{sutton2018reinforcement}.
    Given that the joint policy $\boldsymbol{\pi}$ is an equivariant map, it follows that $p^{\boldsymbol{\pi}}(s_1 \rightarrow s_T, T) = p^{\boldsymbol{\pi}}(s_1^\sigma \rightarrow s_T^\sigma, T)$.
\end{thm}
\begin{proof}
    It follows from our assumption of agent homogeneity that permuting a state $\sigma \cdot s_t$, which (from Theorem \ref{thm_policy_eqv_map}) permutes action selection $\sigma \cdot \boldsymbol{a}_t$, also permutes the environment's transition probabilities:
    \begin{equation*}
        P(s_{t+1}|s_t, \boldsymbol{a}_t) = P(s_{t+1}^\sigma|s_t^\sigma, \boldsymbol{a}_t^\sigma)
    \end{equation*}
    This is because, from the environment's perspective, a state-action pair is indistinguishable from the state-action pair generated by the same agents after swapping their positions and selected actions. Assuming a uniform distribution of start-states $P_\emptyset$, we also have $P_\emptyset(s_1) = P_\emptyset(s_1^\sigma)$.
     Recall the probability of a trajectory from Equation \eqref{eq_traj_prob}. Given the equivariant function $\boldsymbol{\pi}$ and the two equalities above, it follows that:
    \begin{align*}
        P_\emptyset(s_1)\prod_{t=1}^T  & P(s_{t+1} | s_t, a_t)\boldsymbol{\pi}(a_t|s_t) \\
        &= P_\emptyset(s_1^\sigma)\prod_{t=1}^T  P(s_{t+1}^\sigma | s_t^\sigma , a_t^\sigma )\boldsymbol{\pi}(a_t^\sigma|s_t^\sigma)
    \end{align*}
    We can represent the probability of a trajectory as a single transition from initial state $s_1$ to terminal state $s_T$ by marginalizing out the intermediate states, so it follows that:
    \begin{align*}
        &p^{\boldsymbol{\pi}}(s_1 \rightarrow s_T, T)\\ 
        &= \int_{s_1} \cdots \int_{s_{T-1}} P_\emptyset(s_1)\prod_{t=1}^T  P(s_{t+1} | s_t, a_t)\boldsymbol{\pi}(a_t|s_t) \\
        &= \int_{s_1^\sigma} \cdots \int_{s_{T-1}^\sigma} P_\emptyset(s_1^\sigma)\prod_{t=1}^T  P(s_{t+1}^\sigma | s_t^\sigma, a_t^\sigma)\boldsymbol{\pi}(a_t^\sigma|s_t^\sigma) \\
        &= p^{\boldsymbol{\pi}}(s_1^\sigma \rightarrow s_T^\sigma, T)
    \end{align*}
    Thus, the probability of reaching terminal state $s_T$ from initial state $s_1$ is equivalent to the probability of reaching $s_T^\sigma$ from $s_1^\sigma$.
\end{proof}

\begin{thm}
    \label{thm_sym_fair}
    Equivariant policies are exactly fair with respect to team fairness.
\end{thm}
\begin{proof}
    The proof follows directly from Theorem \ref{thm_states_eqv}. Since $p^{\boldsymbol{\pi}}(s_1 \rightarrow s_T, T) = p^{\boldsymbol{\pi}}(s_1^\sigma \rightarrow s_T^\sigma, T)$, the probability of the agents obtaining reward $\boldsymbol{r}$ must be equal to obtaining reward $\boldsymbol{r}^\sigma$. Under the full distribution of initial states, the equality:
    \begin{equation*}
        P[R=\boldsymbol{r} | Z=\boldsymbol{z}] = P[R=\boldsymbol{r}^\sigma | Z=\boldsymbol{z}^\sigma]
    \end{equation*}
    \noindent holds for all $\boldsymbol{r}$ and assignments of sensitive variables $\boldsymbol{z}$. This is only possible if $R \perp Z$ and, therefore, $\textrm{I}(R;Z) = 0$, which meets exact team fairness.
\end{proof}

\subsection{Fairness through equivariance regularization}
Though Fair-E achieves team fairness, it does so in a rigid manner---imposing hard constraints on policy parameters. Fair-E therefore has no choice but to pursue fairness to the fullest extent (and accept the maximum utility trade-off in return). In many cases, it is advantageous to tune the strength of the fairness constraints. For this reason, we propose a soft-constraint version of Fair-E, which we call Fairness through Equivariance Regularization (Fair-ER). Fair-ER is defined by the following regularization objective:
\begin{multline}
    \label{eq_eqv_obj}
    J_{\textrm{eqv}}(\phi_1, ..., \phi_i, ..., \phi_n) \\
    = \underset{s}{\mathbb{E}}[\underset{j \neq i}{\mathbb{E}}[1 - \cos(\pi_{\phi_i}(s) - \pi_{\phi_j}(s))|_{s=s_t}]]
\end{multline}
\noindent which encourages equivariance by penalizing agents proportionally to the amount their actions differ from the actions of their teammates. Using Equation \eqref{eq_eqv_obj}, Fair-ER extends the standard RL objective from Equation \eqref{eq_ddpg_obj} as follows:
\begin{equation}
    J(\phi_i) + \lambda J_{\textrm{eqv}}(\phi_1, ..., \phi_i, ..., \phi_n)
\end{equation}
\noindent where $\lambda$ is a ``fairness control parameter" weighting the strength of equivariance. Differentiating the joint objective with respect to parameters $\phi_i$ produces the Fair-ER policy gradient:
\begin{multline}
    \label{eq_eqv_grad}
    \nabla_{\phi_i} J_{\textrm{eqv}}(\phi_i)
    = \underset{s}{\mathbb{E}}\bigg[\sum_i \frac{1}{N-1} \sum_{j \neq i} \sin(\pi_{\phi_i}(s) \\ - \pi_{\phi_j}(s)) \nabla_{\phi_i} \pi_{\phi_i}(s) |_{s=s_i}\bigg]
\end{multline}
\noindent In this work, Fair-ER is applied to each agent's actor network by optimizing Equation \eqref{eq_eqv_grad} alongside Equation \eqref{eq_ddpg_grad}. Though the above derivations consider stochastic policies, we highlight that Fair-ER is also applicable to deterministic policies and is therefore useful to any multi-agent policy gradient algorithm. We provide further background and a derivation of Equation \eqref{eq_eqv_grad} in Appendix \ref{apdx_derivations_derivations}.

\section{Experiments}
Pursuit-evasion allows us to quantify the performance of emergent team behavior (in terms of both team success and fairness) under variations of ``socio-economic" parameters such as shared objectives and agent skill-level. We therefore use the pursuit-evasion game formalized in Section \ref{sec_preliminaries} to verify our methods. In each experiment, $n{=}3$ pursuer agents are trained in a decentralized manner (each following DDPG) for a total of 125,000 episodes, during which velocity is decreased from $\lvert \vec{v}_p\rvert = 1.2$ to $\lvert \vec{v}_p\rvert = 0.4$. The evader speed is fixed at $\lvert \vec{v}_e\rvert = 1.0$. After training, we test the resulting policies at discrete velocity steps (e.g. $\lvert \vec{v}_p\rvert {=} 1.0$, $\lvert \vec{v}_p\rvert {=} 0.9$, etc), where a decrease in $\lvert \vec{v}_p\rvert$ represents a lesser skilled pursuer. We define the sensitive attribute $z_i$ for each agent $i$ to be a unique identifier of that agent (i.e. $z_i=[0,0,1], z_i=[0,1,0]$ or $z_i=[1,0,0]$ in the $n=3$ case). Each method is evaluated in terms of both utility---through traditional measures of performance such as success rate---and fairness---through the team fairness measure proposed in Section \ref{sec_team_fairness}.

Our evaluation proceeds as follows: first, we study the role of mutual reward in coordination by comparing policies trained with mutual reward to those trained with individual reward. Next, we show that naive mutual reward maximization results in high utility at the expense of fairness. We then show the efficacy of our proposed solution, Fair-E, in resolving these fairness issues. Finally, we evaluate our soft-constraint method, Fair-ER, in balancing fairness and utility.

\subsection{Importance of mutual reward}
We train pursuer policies with decentralized DDPG under conditions of either mutual or individual reward. In the mutual reward condition, pursuers share in the success of their teammates, each receiving the sum of the reward vector $\boldsymbol{r}$. In the individual reward condition, a pursuer is only rewarded if it captures the prey itself, which makes the pursuit-evasion task competitive. The results are shown in Figure \ref{fig_ind_vs_mut}, where utility is the capture success rate of the multi-agent team.

We find that pursuers trained with mutual reward significantly outperform those trained with individual reward. Mutual reward pursuers maintain their performance even as speed drops to $\lvert \vec{v}_p\rvert=0.5$; which is only half of the evader's speed. Under individual reward, performance drops off quickly for $\lvert \vec{v}_p\rvert \leq 1.0$. The velocity $\lvert \vec{v}_p\rvert=1.0$ represents a crucial turning-point in pursuit-evasion---it is the point at which a straight-line chase towards the prey no longer works. These results show that, without mutual reward, the pursuers are not properly incentivized to work together and therefore do not develop a coordination strategy that is any better than a greedy individual pursuit of the evader. Thus, we confirm that mutual reward (a single, shared objective) is vital to coordination.
\label{sec_experiments}
\begin{figure}[]
    \centering
    \makebox[\linewidth][c]{\includegraphics[width=0.9\linewidth]{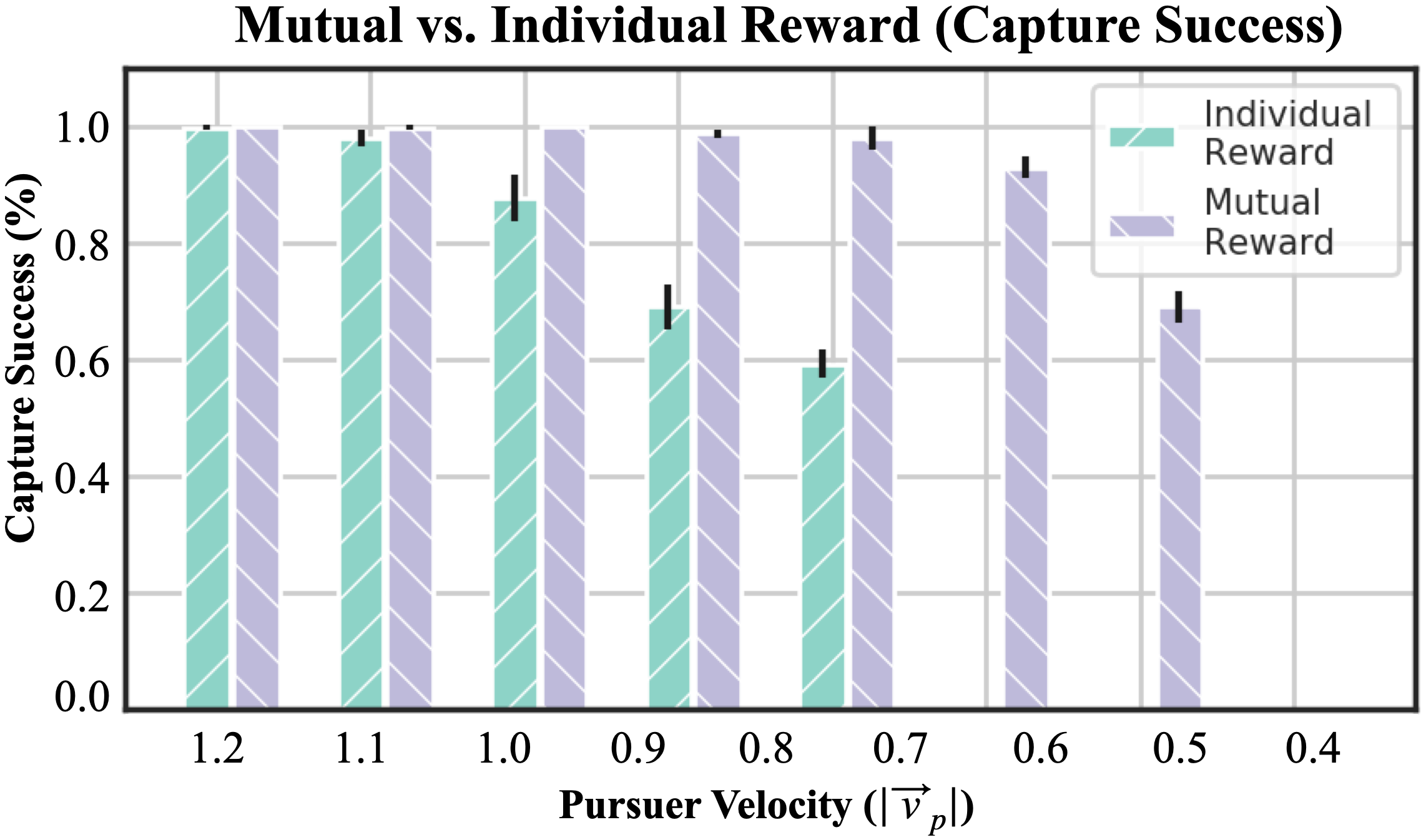}}
    \caption{Performance of policies trained with individual vs. mutual reward. As pursuer velocity decreases, the pursuit-evasion task requires more sophisticated coordination.}
    \label{fig_ind_vs_mut}
\end{figure}
\subsection{Fair outcomes with Fair-E}
\begin{figure*}[t!]
    \centering
    \includegraphics[width=0.95\textwidth]{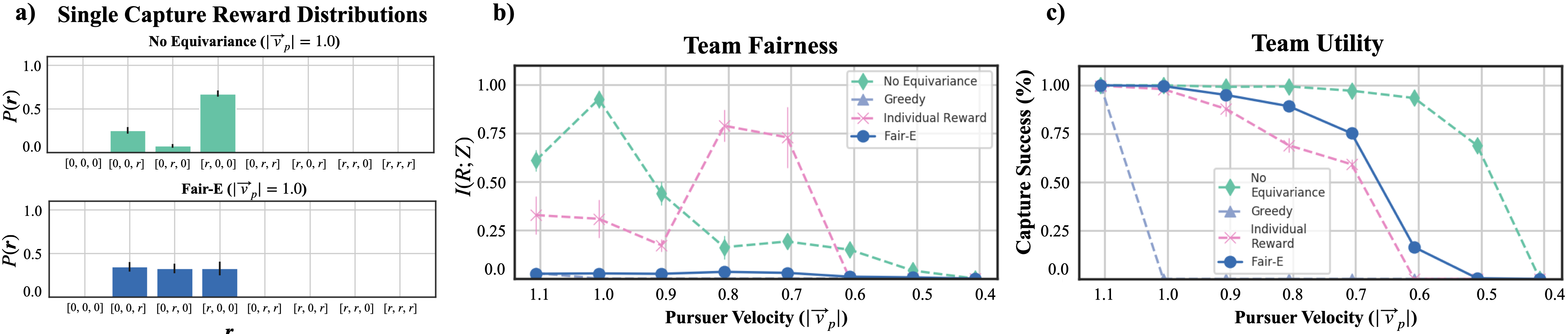}
    \caption{Quantitative comparison of performance for equivariant (i.e. Fair-E) vs. non-equivariant policies. \textbf{a)} Distribution of reward vectors for both strategies at the pursuer velocity $\lvert \vec{v}_p\rvert=1.0$. Non-equivariant policies (top), which learn strategies that push captures towards one agent, yield highly uneven reward distributions. Fair-E policies (bottom) learn to spread captures amongst teammates equally, resulting in even reward distributions. \textbf{b)} Team fairness scores for both strategies (lower better). Fair-E policies yield fairer outcomes than non-equivariant policies across all pursuer velocity (i.e. agent skill) levels. Note that the curve for the Greedy strategy, which is fair despite its low utility, is tucked behind the Fair-E curve. \textbf{c)} Team utility (i.e. capture success) achieved by both strategies (higher better). As pursuer velocity decreases, non-equivariant policies outperform Fair-E by a wide margin, indicating that a fairness-utility trade-off exists.}
    \label{fig_ind_vs_sym}
\end{figure*}
Though mutual reward incentivizes efficient team coordination, it does not stipulate \textit{how} agents should coordinate. To study the nature of the resulting strategy, we examine the distribution of reward vectors obtained by the pursuers over 100 test-time trajectories (averaged over five random seeds each). As shown in Figure \ref{fig_ind_vs_sym}a (top), in which we plot reward vector assignments for captures involving only one pursuer, the pursuer team discovers an unfair strategy---the majority of captures are accounted for by a single agent.

The emergence of an unfair strategy reflects the difficulty of the pursuit-evasion setting. As $\lvert \vec{v}_p\rvert$ decreases, the whole pursuer team has to learn to work together to capture the evader, which is a challenging coordination task. The pursuers learn to do this effectively by assigning roles — e.g. in the $n{=}3$ case, two pursuers take supporting roles, shepherding the evader towards the third agent, who is designated the “capturer”. We note that the decision of which agent becomes the capturer is an emergent phenomenon of the system. As $\lvert \vec{v}_p\rvert$ decreases further, such role assignment is not only helpful but necessary for success. Altogether, the results suggest that unconstrained mutual reward maximization prioritizes utility over fairness.

Our proposed solution, Fair-E, directly combats these fairness issues. Figure \ref{fig_ind_vs_sym}a (bottom) shows the distribution of reward vectors obtained by agents trained with Fair-E. Due to equivariance, Fair-E yields much more evenly distributed rewards. To further quantify these gains, we compare team fairness for both strategies over a variety of skill levels (i.e. $\lvert \vec{v}_p\rvert$ values). The results, shown in Figure \ref{fig_ind_vs_sym}b, confirm that Fair-E achieves much lower $I(R;Z)$ and, therefore, higher team-fairness. Note that, when $\lvert \vec{v}_p\rvert < 0.9$, $I(R;Z)$ is low for non-equivariant pursuers as well. This is an artifact of team fairness---as $\lvert \vec{v}_p\rvert$ decreases, capture success inevitably decreases as well, which is technically a fairer, albeit less desirable, outcome (all agents share equitably in failure). Nevertheless, Figure \ref{fig_ind_vs_sym}b serves as empirical evidence to backup our theoretical result from Section \ref{sec_fairness_equivariance} that Fair-E meets the demands of team fairness.

Despite achieving fairer outcomes, Fair-E is subject to drops in utility as $\lvert \vec{v}_p\rvert$ decreases (see Figure \ref{fig_ind_vs_sym}c). The utility curve for Fair-E drops precipitously for agent skill $\lvert \vec{v}_p\rvert < 0.9$; much faster than the drop-off for pursuers with no equivariance. This is because Fair-E directly prevents role assignment. By hard-constraining each agent's policy, Fair-E enforces $\pi_i(s_t) {=} \pi_j(s_t)$, whereas role assignment requires $\pi_i(s_t) {\neq} \pi_j(s_t)$ for $i {\neq} j$. We emphasize role assignment as key to this result, as parameter-sharing has been shown to be helpful in problem domains that do not require explicit role assignment \cite{baker2019emergent}. In the context of fairness, however, these results indicate that Fair-E will always elect to give up utility to preserve fairness.

For completeness, we also show results for the policies learned with individual reward (from Figure \ref{fig_ind_vs_mut}) and a hand-crafted greedy control strategy in which each pursuer runs directly towards the evader. Note that greedy policies are equivariant---by definition, agents will select similar actions in similar states---but demonstrate no coordination. For this reason, greedy policies have high fairness, but very low utility. Utility follows the same pattern for individual reward policies. Interestingly though, individual reward policies become less fair between $\lvert \vec{v}_p\rvert=0.9$ and $\lvert \vec{v}_p\rvert=0.6$, before tapering off as performance decreases. We defer further discussion of this finding, as well as details regarding the computation of the team fairness score, $I(R;Z)$, and the hand-crafted greedy control baseline to Appendix \ref{apdx_experimental_details}.

\subsection{Modulating fairness with Fair-ER}
Unlike Fair-E, Fair-ER allows policies to balance fairness and utility dynamically. Intuitively, this is because Fair-ER incentivizes policy equivariance through the regularization objective from Equation \eqref{eq_eqv_obj}, while still allowing agents to update their own individual policy parameters (unlike Fair-E). Therefore, the value of the fairness control weight $\lambda$ will dictate how much each agent values fairness vs. utility. 

To study the effectiveness of this method, we trained Fair-ER agents in increasingly difficult environments (by decreasing pursuer velocity $\lvert \vec{v}_p\rvert$) while modulating the fairness control parameter $\lambda$. The effect of $\lambda$ on policy training is shown in Figure \ref{fig_fairer_training}. The results show that, for $\lambda {\leq} 0.5$, Fair-ER is successful in bridging the performance gap between non-equivariant policies ($\lambda{=}0.0)$ and Fair-E (black line). Importantly, we show that it is also possible to over-constrain the system so that it actually performs worse than Fair-E (e.g. $\lambda{=}1.0$). This indicates that, though Fair-ER can mitigate the drops in performance described above, the regularization parameter $\lambda$ must be tuned appropriately.

We also performed the same test-time analysis as described for Fair-E in the previous subsection. Figure \ref{fig_fairness_utility} shows the effect of $\lambda$ on both fairness ($I(R;Z)$) and utility (capture success). For each skill level, increasing $\lambda$ allows Fair-ER to fine-tune the balance between fair and unfair policies, achieving the highest utility possible under its given constraints. We find that, with high values of $\lambda$ (e.g. $\lambda=0.9$), Fair-ER prioritizes fairness over utility and performs in-line with (or worse than) Fair-E---achieving fair outcomes, even at the expense of utility. When $\lambda$ is in the range $\lambda=0.5$ to $\lambda=0.1$, Fair-ER withstands a drop in utility until $\lvert \vec{v}_p\rvert=0.7$ by giving up small amounts of fairness. Therefore, we find evidence that learning multi-agent coordination strategies with Fair-ER simultaneously maintains \textit{higher utility} than Fair-E while achieving \textit{higher fairness} than non-equivariant learning. Overall, tuning the fairness weight $\lambda$ allows us to directly control the strength of the fairness constraints imposed on the system, enabling Fair-ER to modulate fairness to the needs of the task.

\subsection{Fairness-utility trade-off}
As we saw in Figure \ref{fig_ind_vs_sym}, the hard constraints that Fair-E places on each agent's policy creates an inherent commitment to achieving fair outcomes at the expense of utility (or reward). In this section, we examine the extent to which the fairness-utility trade-off exists for Fair-ER across all agent skill levels. For each agent skill level (i.e. $\lvert \vec{v}_p\rvert$ value), we computed both the fairness (through team fairness $I(R;Z)$) and utility (through capture success) scores achieved by multi-agent coordination strategies learned for $\lambda$ values in the range $\lambda \in \{0.0, 1.0\}$ over 100 test-time trajectories (averaged over five random seeds each). The results of this experiment are shown in Figure \ref{fig_fairness_utility}. 

Unlike many prior studies in both traditional multi-agent fairness settings \cite{okun2015equality, le1990equity, bertsimas2012efficiency, joe2013multiresource, bertsimas2011price} and prediction-based settings \cite{corbett2017algorithmic, zhao2019inherent}, we find that it is not always the case that fairness must be traded for utility. With Fair-ER, fairness comes for little to no cost in utility until $\lvert \vec{v}_p\rvert=0.8$. This means that, when each agent operates at a high skill level, requiring each agent in the multi-agent team to shift towards an equivariant policy (which yields fair results) does not cause coordination of the larger multi-agent team to break down. When $\lvert \vec{v}_p\rvert<0.8$, however, utility drops quickly for larger values of $\lambda$. This indicates that, when agent skill decreases (or the task becomes more complex relative the agents' current skill level), unfair strategies such as role assignment are the only effective way to maintain high levels of utility. Overall, these results serve as empirical evidence that, in the context of cooperative multi-agent tasks, fairness is inexpensive, so long as the task is easy enough (i.e. agent skill is high enough). As task difficulty increases, fairness comes at an increasingly steep cost. To the best of our knowledge, such a characterization of the fairness-utility trade-off for multi-agent settings has not been illustrated in the fairness literature.
\begin{figure}[t!]
    \centering
    \makebox[\linewidth][c]{\includegraphics[width=0.95\linewidth]{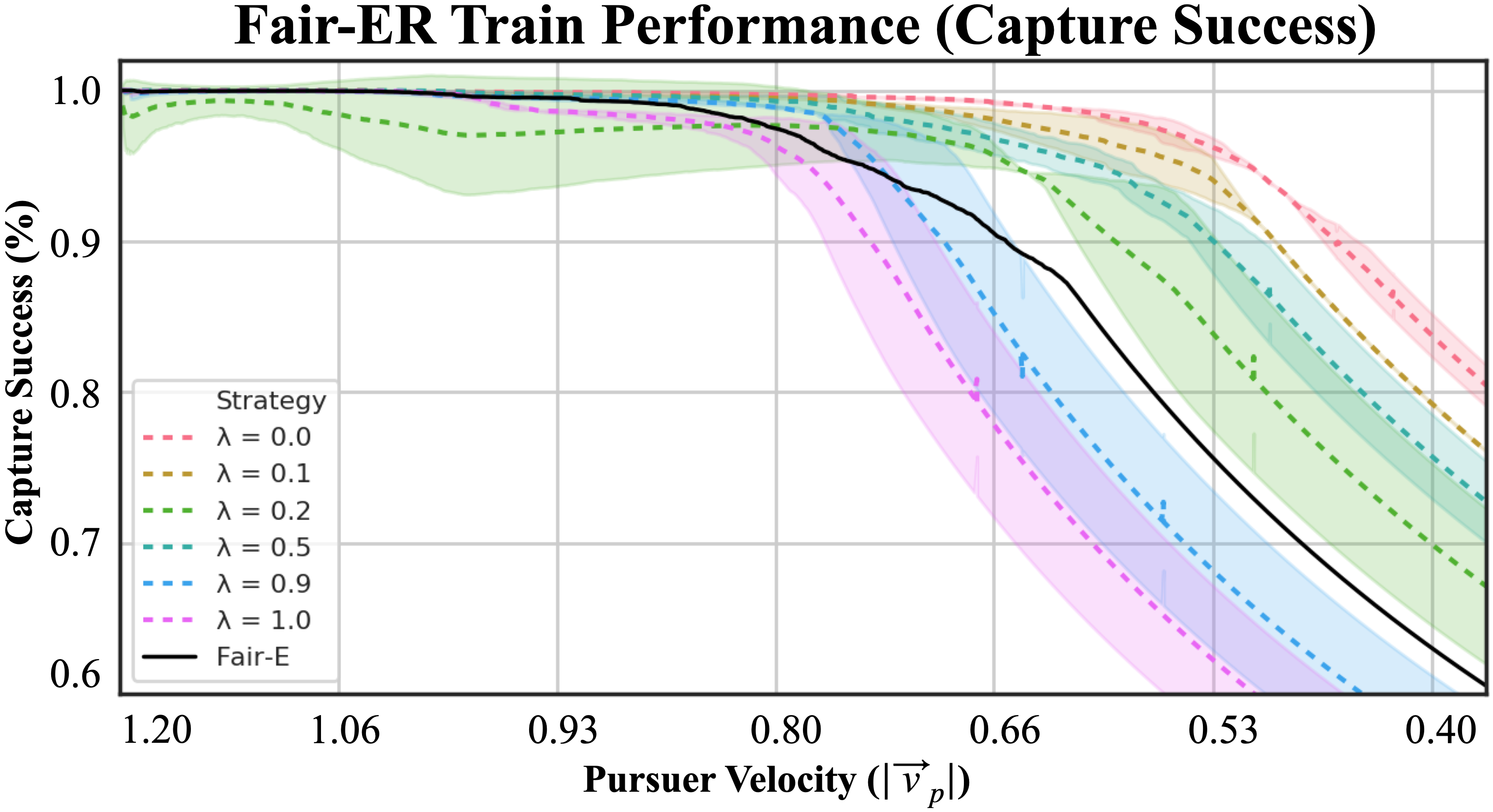}}
    \caption{Effect of the equivariance control parameter $\lambda$ on policy learning with Fair-ER. Increasing $\lambda$ yields fairer policies, but causes performance to decay more quickly in difficult environments (lower agent skill). The black line represents Fair-E's performance for comparison.}
    \label{fig_fairer_training}
\end{figure}

\section{Conclusion}
\label{sec_conclusion}
Multi-agent learning holds promise for helping AI researchers, economic theorists, and policymakers alike better evaluate real-world problems involving social structures, taxation, policy, and economic systems broadly. This work has focused on one such problem; namely, fairness in cooperative multi-agent settings. In particular, we have demonstrated that fairness issues arise naturally in cooperative, single-objective multi-agent learning problems. We have shown that our proposed method, equivariant policy optimization (Fair-E), mitigates such issues. We have also shown that soft constraints (Fair-ER) lower the cost of fairness and allow the fairness-utility trade-off to be balanced dynamically. Moreover, we have presented novel results regarding the fairness-utility trade-off for cooperative multi-agent settings; identifying a connection between agent skill and fairness. In particular, we showed that fairness comes for free when agents are highly-skilled, but becomes increasingly expensive for lesser-skilled agents.

This work represents a first step towards understanding the core factors underlying fairness and multi-agent learning in environments where team dynamics and coordination are important for task success. There are a number of exciting avenues of future work that build upon these initial ideas. First, ongoing work is investigating cooperative multi-agent fairness in more complex domains (e.g. video games, simulated economic societies). Moreover, there is room to explore indirect or backdoor causal paths between sensitive and target variables in the context of multi-agent teams, which warrant connecting additional interpretations of fairness (e.g. causal fairness) to cooperative multi-agent settings.
\begin{figure}[t!]
    \centering
    \makebox[\linewidth][c]{\includegraphics[width=0.99\linewidth]{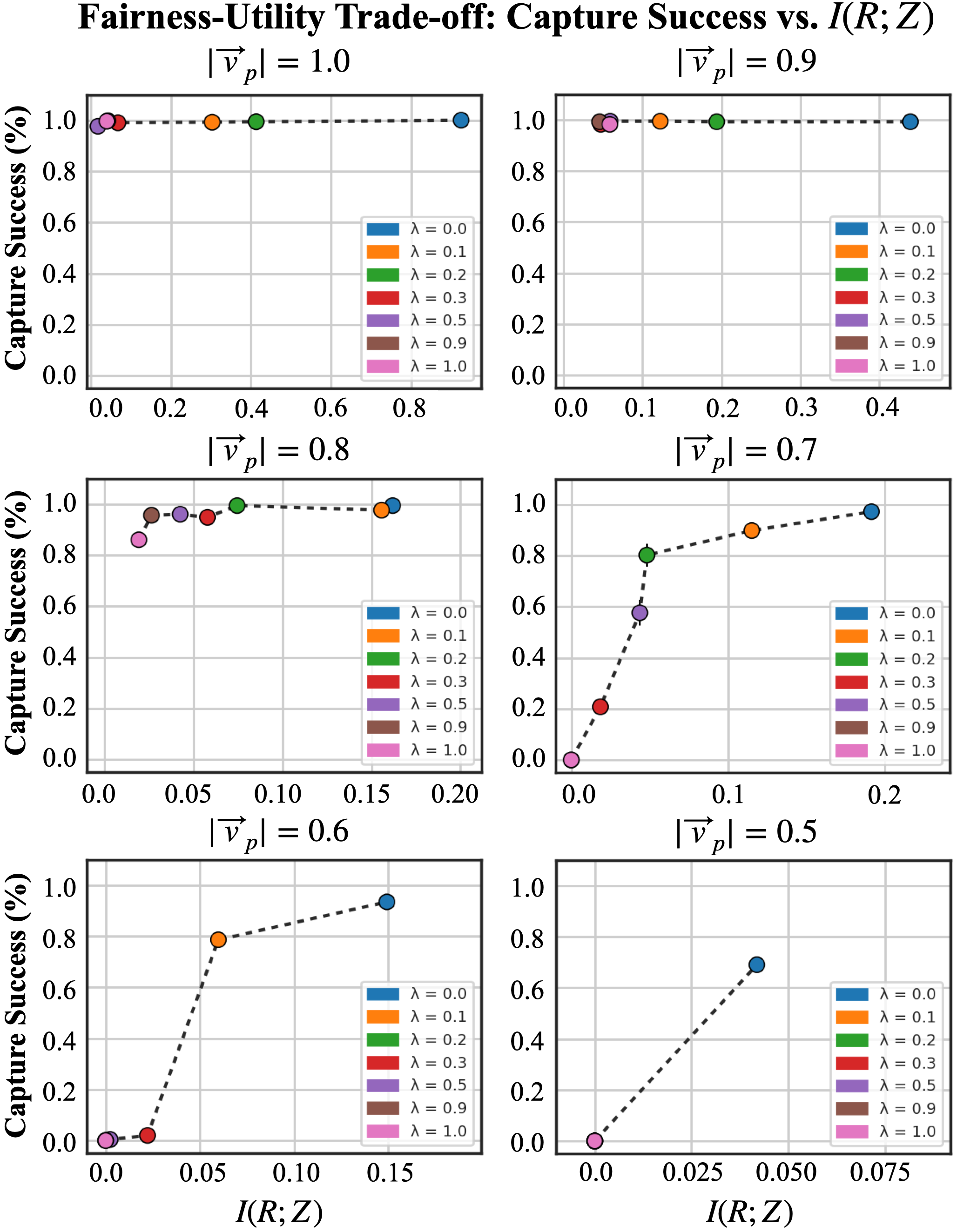}}
    \caption{Fairness vs. utility comparisons for Fair-ER trained with various values of equivariance regularization $\lambda$ and environment difficulties. Note that $\lambda=0.0$ is equivalent to no equivariance.}
    \label{fig_fairness_utility}
\end{figure}

\section*{Acknowledgments}
We thank the reviewers for their valuable feedback. This research was supported by NSF awards CCF-1522054 (Expeditions in computing), AFOSR Multidisciplinary University Research Initiatives (MURI) Program FA9550-18-1-0136, AFOSR FA9550-17-1-0292, AFOSR 87727, ARO award W911NF-17-1-0187 for our compute cluster, and an Open Philanthropy award to the Center for Human-Compatible AI.

\bibliography{aaai22}

\clearpage
\appendix

\section{Deterministic fairness through equivariance}
\label{apdx_deterministic_proofs}
We review the proofs from Section \ref{sec_fairness_equivariance} in the context of deterministic policies. We assume that agents are homogeneous in their non-sensitive variables. For $n$ agents, each with an individual policy $\mu_{\phi_i}$, let $\boldsymbol{\mu} = \{\mu_{\phi_1}, ..., \mu_{\phi_n} \}$ be the joint policy representing the team. Recall that, in the symmetric case, $\phi_1=\phi_2=...=\phi_n$ and $\omega_1=\omega_2=...=\omega_n$.

\begin{thm}
    \label{thm_deterministic_policy_eqv_map}
    If individual policies $\mu_{\phi_i}$ are symmetric, then the joint policy $\boldsymbol{\mu} = \{\mu_{\phi_1}, ..., \mu_{\phi_n} \}$ is an equivariant map.
\end{thm}
\begin{proof}
    Let $\sigma$ be a permutation operator that, when applied to a vector (such as a state $s_t$ or joint action $\boldsymbol{a}_t$), produces the permuted vector ($\sigma \cdot s_t {=} s^\sigma_t$ or $\sigma \cdot \boldsymbol{a}_t {=} \boldsymbol{a}^\sigma_t$, respectively). Under parameter symmetry (i.e. $\phi_1${=}$\phi_2{=}{\cdots}{=}\phi_n$), we have:
    \begin{equation*}
        \boldsymbol{\mu}(\sigma \cdot s) = \boldsymbol{\mu}(s^{\sigma}) \\
        = \boldsymbol{a}^\sigma \\
        = \sigma \cdot \boldsymbol{a} \\
        = \sigma \cdot \boldsymbol{\mu}(s)
    \end{equation*}
    \noindent where the commutative relationship $\boldsymbol{\mu}(\sigma \cdot s) = \sigma \cdot \boldsymbol{\mu}(s)$ implies that $\boldsymbol{\mu}$ is an equivariant map.
\end{proof}

\begin{thm}
    \label{thm_deterministic_states_eqv}
    Let $p^{\boldsymbol{\mu}}(s \rightarrow s', k)$ be the probability of transitioning from $s$ to $s'$ in $k$ steps.
    Given that the joint policy $\boldsymbol{\mu}$ is an equivariant map, it follows that $p^{\boldsymbol{\mu}}(s_1 \rightarrow s_T, T) = p^{\boldsymbol{\mu}}(s_1^\sigma \rightarrow s_T^\sigma, T)$.
\end{thm}
\begin{proof}
    It follows from agent homogeneity that permuting a state $\sigma \cdot s_t$, which in turn permutes action selection $\sigma \cdot \boldsymbol{a}_t$ (from Theorem \ref{thm_deterministic_policy_eqv_map}), also permutes the environment's transition probabilities:
    \begin{equation*}
        P(s_{t+1}|s_t, a_t) = P(s_{t+1}^\sigma|s_t^\sigma, \boldsymbol{a}_t^\sigma)
    \end{equation*}
    This is because, from the environment's perspective, a state-action pair is indistinguishable from the state-action pair generated by the same agents after swapping their positions and selected actions. Assuming that the full distribution of start-states $P_\emptyset$ is uniform, we also have $P_\emptyset(s_1) = P_\emptyset(s_1^\sigma)$.
     Recall the probability of a trajectory from Equation \eqref{eq_traj_prob}. Given the equivariant function $\boldsymbol{\mu}$ and the two equalities above, it follows that:
    \begin{align*}
        P_\emptyset(s_1)\prod_{t=1}^T & P(s_{t+1} | s_t, \boldsymbol{a}_t)|_{\boldsymbol{a}_t = \boldsymbol{\mu}(s_t)} \\
        &= P_\emptyset(s_1^\sigma)\prod_{t=1}^T  P(s_{t+1}^\sigma | s_t^\sigma , \boldsymbol{a}_t^\sigma )|_{\boldsymbol{a}_t^\sigma  = \boldsymbol{\mu}(s_t^\sigma )}
    \end{align*}
    Note that the following properties hold for $p^{\boldsymbol{\mu}}(s \rightarrow s', k)$:
    \begin{itemize}
        \item $p^{\boldsymbol{\mu}}(s \rightarrow s, 0) = 1$
        \item $p^{\boldsymbol{\mu}}(s \rightarrow s', 1) = p(s_{t+1}=s' | s_t=s, a_t) | a_t=\mu(s_t)$,
        \item $p^{\boldsymbol{\mu}}(s \rightarrow x, T) = \int_{s'} p^\mu(s \rightarrow s', T-1) p^\mu(s' \rightarrow x, 1)$
    \end{itemize}
    We can therefore represent the probability of a trajectory as a single transition from initial state $s_1$ to terminal state $s_T$ by marginalizing out the intermediate states and show that:
    \begin{align*}
        &p^\mu(s_1 \rightarrow s_T, T) \\
        &= \int_{s_{T-1}} p^\mu(s_1 \rightarrow s_{T-1}, T-1) p^\mu(s_{T-1} \rightarrow s_T, 1) \\
        &\begin{multlined} = \int_{s_{T-2}} p^\mu(s_1 \rightarrow s_{T-2}, T-2) \big[p^\mu(s_{T-2} \rightarrow s_{T-1}, 1) \\ p^\mu(s_{T-1} \rightarrow s_T, 1)\big] \end{multlined}\\
        &= \cdots \\
        &= \int_{s_1} \cdots \int_{s_{T-1}} P_\emptyset(s_1)\prod_{t=1}^T  P(s_{t+1} | s_t, a_t)|_{a_t = \mu(s_t)} \\
        &= \int_{s_1^\sigma} \cdots \int_{s_{T-1}^\sigma} P_\emptyset(s_1^\sigma)\prod_{t=1}^T  P(s_{t+1}^\sigma | s_t^\sigma, a_t^\sigma)|_{a_t^\sigma = \mu(s_t^\sigma)} \\
        &= \cdots \\
        &= p^\mu(s_1^\sigma \rightarrow s_T^\sigma, T)
    \end{align*}
    Thus, the probability of reaching terminal state $s_T$ from initial state $s_1$ is equivalent to the probability of reaching $s_T^\sigma$ from $s_1^\sigma$.
\end{proof}

\begin{thm}
    \label{thm_deterministic_sym_fair}
    Equivariant deterministic policies are exactly fair.
\end{thm}
\begin{proof}
    Exact same as Theorem \ref{thm_sym_fair}.
\end{proof}

\section{Derivations}
\label{apdx_derivations_derivations}
\subsection{Gradient of fairness objective}
\label{apdx_derivations_sym}
We present a derivation of the equivariant objective gradient from Equation \eqref{eq_eqv_grad}:
\begin{align*}
    \frac{\partial J_{\textrm{eqv}}}{\partial \phi_i} &\approx \frac{\partial}{\partial \phi_i} \frac{1}{M} \sum_i^M \frac{1}{N-1} \sum_{j \neq i}^{N-1}1 - \cos(\mu_{\phi_i}(s) - \mu_{\phi_j}(s))
    \\[4pt]
    &\begin{multlined} =\frac{1}{M} \sum_i^M \frac{1}{N-1} \sum_{j \neq i}^{N-1} - (-\sin(\mu_{\phi_i}(s) \\ \qquad \qquad - \mu_{\phi_j}(s))) \frac{\partial}{\partial \phi_i} (\mu_{\phi_i}(s) - \mu_{\phi_j}(s)) \end{multlined}
    \\[4pt]
    &\begin{multlined} =\frac{1}{M} \sum_i^M \frac{1}{N-1} \sum_{j \neq i}^{N-1} \sin(\mu_{\phi_i}(s) \\ \qquad \qquad - \mu_{\phi_j}(s)) \bigg(\frac{\partial}{\partial \phi_i} \mu_{\phi_i}(s) - \frac{\partial}{\partial \phi_i} \mu_{\phi_j}(s) \bigg)
    \end{multlined}
    \\[4pt]
    &\begin{multlined} = \frac{1}{M} \sum_i^M \frac{1}{N-1} \sum_{j \neq i}^{N-1} \sin(\mu_{\phi_i}(s) \\ \qquad \qquad - \mu_{\phi_j}(s)) \nabla_{\phi_i} \mu_{\phi_i}(s)
    \end{multlined}
\end{align*}

\subsection{Why not mean squared error?}
\label{apdx_mse}
We could just have as easily defined $J_{\textrm{eqv}}$ as the mean-squared error (MSE) between $a_i$ and $a_j$, with the objective:
\begin{equation*}
    J_{\textrm{sym}}(\phi_1, ..., \phi_i, ..., \phi_n) = \mathbb{E}_s[ \mathbb{E}_{j \neq i}[(\mu_{\phi_i}(s) - \mu_{\phi_j}(s))^2]]
\end{equation*}
and corresponding gradient:
\begin{multline*}
    \nabla_{\phi_i} J_{\textrm{sym}}(\phi) \approx \frac{1}{M} \sum_i \frac{1}{N-1} \sum_{j \neq i} (\mu_{\phi_i}(s) \\ - \mu_{\phi_j}(s)) \nabla_{\phi_i}\mu_{\phi_i}(s) |_{s=s_i}
\end{multline*}
\noindent However, recall that each $a$ is a heading angle sampled from $A$, which is a circular variable in the range $[0, 2\pi]$ or $[-\pi, \pi]$. Any difference of actions, as in MSE loss, must account for this by applying a modulo operation on top of the difference. Our objective represents angular distance as a single Fourier mode, which yields a convex optimization surface without discontinuity issues on the range of possible headings.

\section{Experimental details}
\label{apdx_experimental_details}

\paragraph{Computing team fairness}
We illustrate how team fairness scores $I(R;Z)$ are computed. Recall that each agent's sensitive attribute $z_i$ is an identity variable that uniquely identifies it from the group (see Section \ref{sec_experiments}). Also note that the vectorial reward vector $\boldsymbol{r}$ serves as a proxy for agent identity---e.g. $\boldsymbol{r} {=} [0,0,1]$ indicates that pursuer $p_3$ captured the evader. We can therefore compute team fairness as the mutual information obtained by the difference of entropies $h(R) - h(R_\textrm{uniform})$, where $R$ is a distribution over team rewards $\boldsymbol{r}$ and $R_\textrm{uniform}$ is a uniform reward distribution (i.e. captures are spread evenly across all pursuers). Computed this way, $I(R;Z)$ represents the extent to which knowing the outcome of the pursuit-evasion task reveals the agents identity, and vice versa. It therefore measures how fairly outcomes are distributed across cooperative teammates.

\paragraph{Greedy control baseline}
Let $q_i$ be the position of an agent $i$ and $U_{\textrm{att}}(q_i)$ be a quadratic function of distance between $q_i$ and a target $q_{\textrm{goal}}$:
\begin{equation}
    \label{eq_attractive_potential}
    U_{\textrm{att}}(q_i, q_{\textrm{goal}}) = \frac{1}{2}k_{\textrm{att}} \, d(q_i, q_{\textrm{goal}})^2
\end{equation}
where $k_{\textrm{att}}$ is an attraction coefficient and $d(,)$ is a measure of distance. Taking the negative gradient $F(q_i) = -\nabla U(q_i)$ yields the following control law for agent $i$'s motion:
\begin{equation}
    F_{\textrm{att}} = -\nabla U_{\textrm{att}}(q_i, q_{\textrm{goal}}) = -k_{\textrm{att}}(q_i - q_{\textrm{goal}})
\end{equation}
\noindent In this work, the environment's action-space is defined in terms of agent headings, so only the \textit{direction} of this force impacts the agents. Setting $q_{\textrm{goal}}$ to be the position of the evader and following Equation \eqref{eq_attractive_potential} at each time-step results in a greedy policy that runs directly towards the evader.

\paragraph{Fairness for individual reward policies}
In Figure \ref{fig_ind_vs_sym}b, we find that policies learned with individual reward exhibit a temporary spike in $I(R;Z)$ from $\lvert \vec{v}_p\rvert=0.9$ to $\lvert \vec{v}_p\rvert=0.6$, indicating that their reward distributions $R$ are less fair. This occurs because, as $\lvert \vec{v}_p\rvert$ decreases, policies learned with individual reward fall into a degenerate state where only one pursuer experiences positive reward from capturing the evader. The other pursuers, never experiencing positive reward, slowly diverge from their previous greedy strategies and become incapable of capturing the evader. Their best strategy becomes hoping that the ``capturer" pursuer captures the evader, which results in a highly skewed capture distribution. Eventually this strategy fails when $\lvert \vec{v}_p\rvert$ drops low enough that all pursuers fail to capture the evader, which causes $I(R;Z)$ to fall again.

\paragraph{Policy learning hyperparameters}
All actors $\mu_\phi$ are trained with two hidden layers of size 128. Critics $Q_\omega$ are trained with three hidden layers of size 128. We use a learning rate of $1\textrm{e}^{-4}$ and $1\textrm{e}^{-3}$ for the actor and critic, respectively, and a gradient clip of 0.5 on both. Target networks are updated with Polyak averaging with $\tau=0.001$. We maintain a buffer $\mathcal{D}$ of length $500000$ and sample batches of size $512$. Finally, we use a discount factor $\gamma=0.99$. All values are the results of standard hyperparamter sweeps.

\paragraph{Experiments}
As described in Section \ref{sec_experiments}, each pursuit-evasion experiment includes $n=3$ pursuer agents and a single evader. The pursuers each train their own policy for a total of 125,000 episodes, during which their velocity is decreased from $\lvert \vec{v}_p\rvert = 1.2$ to $\lvert \vec{v}_p\rvert = 0.4$. The evader speed is fixed at $\lvert \vec{v}_e\rvert = 1.0$. After training, we test the resulting policies at discrete velocity intervals (e.g. $\lvert \vec{v}_p\rvert = 1.2$, $\lvert \vec{v}_p\rvert = 1.1$, etc), where a decrease in pursuer velocity represents a greater ``difficulty level" for the pursuers. Test-time performance, such as is shown in Figures \ref{fig_ind_vs_mut}b, \ref{fig_ind_vs_sym}c, \ref{fig_ind_vs_sym}d, and \ref{fig_fairness_utility} is averaged across 100 independent trajectories from five different random seeds each. All experiments leveraged an Nvidia GeForce GTX 1070 GPU with 8GB of memory.

\section{Additional information}
\label{apdx_additional_info}

\paragraph{Assets}
The pursuit-evasion environment used in this work is an extension of the environment introduced by \citet{lowe2017multi}. The original environment is open-sourced on Github under the MIT license. We cited the authors accordingly in the main text. Since we only change the reward function, we do not include the pursuit-evasion environment as a new asset in the supplementary material. Instead, we point to the original repo\footnote{https://github.com/openai/multiagent-particle-envs}. None of the assets used in this work contain personally identifiable information or offensive content.

\paragraph{Limitations} In our Fair-ER learning strategy, the fairness control parameter $\lambda$ allows the strength of equivariance regularization to be tuned to achieve the desired fairness-utility trade-off for the task at hand. Despite these benefits, tuning $\lambda$ is currently a manual process. An ideal solution, which is the subject of future work, would automatically find the appropriate $\lambda$ value for the desired fairness-utility trade-off. We highlight, however, that many regularization techniques in machine learning (most notably, $L_1$ and $L_2$ regularization) require manual tuning as well.

Moreover, our method does not directly address non-stationarity in multi-agent RL. Non-stationarity can cause high variance policy gradients and learning instability. This did not significantly impact training in our case, but possible solutions, such as centralized training \cite{foerster2016learning, lowe2017multi} will be considered in future work. Though we primarily evaluate our method in the context of pursuit-evasion games, we highlight that both Fair-E and Fair-ER are task-agnostic---they can be applied to any multi-agent environment.
\end{document}